\documentclass{article}

\usepackage[utf8]{inputenc}
\usepackage[english]{babel}
\usepackage[T1]{fontenc}
\usepackage{amsmath,amssymb}
\usepackage{bm}
\usepackage{amsthm}
\usepackage{mathtools}
\usepackage{pdflscape}
\usepackage[hidelinks]{hyperref}
\usepackage{dsfont, mathrsfs}
\usepackage{graphicx}
\usepackage{mathrsfs}
\usepackage{wrapfig}
\usepackage{microtype}
\usepackage{subfigure}
\usepackage{booktabs} 
\graphicspath{{Figures/}}

\usepackage{algorithm}
\usepackage[noend]{algpseudocode}
\usepackage[numbers, compress, sort]{natbib}
\usepackage{xfrac}
\usepackage{sidecap}

\newtheorem{theorem}{Theorem}[section]
\newtheorem{lemma}[theorem]{Lemma}
\newtheorem{proposition}[theorem]{Proposition}
\newcounter{principles}
\newtheorem{principle}[principles]{Principle}

\newtheorem{definition}{Definition}

\newcommand\independent{\protect\mathpalette{\protect\independenT}{\perp}}
\def\independenT#1#2{\mathrel{\rlap{$#1#2$}\mkern2mu{#1#2}}}

\usepackage{times}

\title{Reparametrization Invariance for non-parametric Causal Discovery}

\author{Martin Jørgensen and Søren Hauberg\\
\small{\href{mailto:marjor@dtu.dk}{marjor@dtu.dk}\quad\href{mailto:sohau@dtu.dk}{sohau@dtu.dk}}\\
\small{Technical University of Denmark}}

\date{}

\let\cite\citep 
\begin{document}

\maketitle

\begin{abstract}
	Causal discovery estimates the underlying physical process that generates the
	observed data: does $X$ cause $Y$ or does $Y$ cause $X$? Current methodologies use structural conditions to turn the causal query into a statistical query, when only observational data is available. But what if the these statistical queries are sensitive to causal invariants? This study investigates one such invariant: \emph{the causal relationship between $X$ and $Y$ is invariant to the marginal distributions of $X$ and $Y$}. We propose an algorithm that use a non-parametric estimator that is robust to changes in the marginal distributions. This way we may marginalize the marginals, and inspect what relationship is intrinsically there. The resulting causal estimator is competitive with current methodologies and has high emphasis on the \emph{uncertainty} in the causal query; an aspect just as important as the query itself.
\end{abstract}

\section{Introduction}
Determining causal relationships is a constant challenge, and the ultimate goal of the natural sciences. 
The gold standard for establishing such relationships is \emph{intervention studies},
where the physical state of a system is manually modified to determine whether this
changes the system behavior. Such experiments are, however, often infeasible as
the interventions can be unethical, physically impossible, expensive and so forth.
This begs the question of whether causal relationships can be estimated from data in
a systematic manner.
Most work in this direction has been for high-dimensional data used to estimate
directed acyclic graphs (DAGs), but in recent years the most simple of these,
the two-vertex DAG, has gained more attention. The methods for determining
these causal bindings go under the name of \emph{causal discovery}, and the usual
approach is to assume some structural equation model, and probabilistically
verify its existence.

By assuming a particular model, it becomes possible to establish conditions
under which the causal direction is unique, thereby providing a
formalism to the causal question. From a practical point of view, this formalism
is, however, only useful when the structural model assumption is known to be true, which is
seldom the case.

In this paper, we explore the case of bivariate causal inference when model
assumptions are challenged by shifts in marginal distributions. We propose an estimator based on comparing regression errors, as in \citet{bloebaum}, but in a non-parametric way. This provides an estimator that is more robust to these distributional shift than well-known methods for bivariate causal discovery, while staying on-par in performance.
\subsection{Related Work}\label{sec:related}\vspace{-2mm}
In his seminal work, \citet{pearl} introduced causal inference for high-dimensional
observational data, phrased as the estimation of a causal structure. This is a DAG, where random variables are nodes and an edge $X \rightarrow Y$
indicates that $X$ is a (direct) cause of $Y$. Given more than three variables, such edges
can be estimated through conditional independence tests, e.g. an edge between
$X$ and $Y$ can be discarded if they are independent conditioned on a third variable $Z$.
This idea, however, breaks down in the bivariate case, which is the main focus of the present paper.

In the bivariate case, one usually must impose assumptions that break the symmetry of correlation. 
This is achieved by assuming two models --- one for $X \rightarrow Y$ and another
for $Y \rightarrow X$ --- and choosing among these either by
1) verifying exactly one of the underlying models or 2) proposing a score/complexity measure for
choosing the simplest model following the principles of Occam's razor.

\paragraph{Model Verification \cite{nonlinear, lingam, pnl}:}
It is natural to assume an \emph{additive noise model (ANM)} \cite{nonlinear, lingam}, 
i.e.\ $Y = f(X) + N_Y$, where $N_Y\independent X$.
\citet{nonlinear} show that when $f$ is nonlinear, then the true causal
direction can be identified.
Similar results hold when $f$ is linear and the noise is non-Gaussian \cite{lingam}.
However, if the underlying system is not an ANM, the analysis is inapplicable -- e.g. in the presence of hidden confounders.
\citet{pnl} extend the ANM to allow for an unknown bijective mapping of the observations
and show that this structure is identifiable for many joint distributions $\mathbb{P}_{(X,Y)}$.

\paragraph{Model Scoring \cite{CURE, igci, gpi, bloebaum}:}
An intuitive scoring mechanism is to regress $Y$ from $X$ and vice versa and ask
which direction has higher likelihood. This is, e.g., implemented by \citet{gpi}
who propose using a Gaussian Process Latent Variable Model \citep{lawrence2005probabilistic} to handle the noise/latent observations. The chosen causal direction must
then be \emph{biased} towards the prior over the latent points and sensitive to hyperparameters, which is an implicit model
assumption.

\citet{bloebaum} take an approach based on asymmetry of regression error, and show that this asymmetry is coherent with the causal direction under certain assumptions, of which the most important are the independence of the cause and the causal \emph{mechanism} \cite{elements} and that this mechanism is monotonic as a function of the cause. Loosely, they show that when the noise is sufficiently small and $X\rightarrow Y$, then
\begin{equation}\label{bloe}
	\mathbb{E}[\text{Var}(Y|X)]\, \leq\, \mathbb{E}[\text{Var}(X|Y)].
\end{equation}
They quantify these measures by parametric regression.

\citet{igci} propose the \emph{Information Geometric Causal Inference} (IGCI) scoring mechanism.
This is derived from the assumption
that data is noise free, i.e. $Y = f(X)$, and on the postulate that the true causal mechanism
$f$ is independent of the cause $X$. This is realized by non-parametrically
estimating the expected log-derivative of $f$:
\begin{equation}
	\mathbb{E}[\log |f'|] \approx \frac{1}{N-1}\sum_{i=1}^{N-1}\log\frac{|y_{i+1}-y_i|}{x_{i+1}-x_i},
\end{equation}
where $x_{i+1}>x_i$ for $i=1,\ldots,N-1$, and both $X$ and $Y$ have been preprocessed
to make them comparable, i.e. standardized wrt.\ a Gaussian or a uniform base measure.
The direction with the smallest log-derivative is then chosen as being causal.
While this mechanism provides no guarantees in the presence of noise, IGCI has
been successful on real world data; in Sec.~\ref{robustness} we, however, demonstrate that this
success is likely due to a bias in the studied benchmark data.

\subsection{Causal invariant}
As seen above, current causal inference propose one
model for each causal direction, and then select among them.
This begs the questions, \emph{what if the data does not support either model?} and \emph{can causal relationships be discovered without restrictive model assumptions?} If we believe that the causal and probabilistic domains abide by different rules, then our causal estimators should follow other paradigms than model verification/selection. We can think of this as \emph{model-bias}:  many existing methods are too sensitive to distributional and structural restrictions of probabilistic models. By this we mean that the \emph{hypothesis} of causality is tested in a domain sensitive to marginal distributions and structural equations.

We recap the basic definition of causality as expressed
by the \emph{do-calculus} \citep{pearl}.
\begin{definition}\label{def:do}
	If for some $x\neq \hat{x}$, we have that $\mathbb{P}(Y|\mathrm{do}(x))\neq \mathbb{P}(Y|\mathrm{do}(\hat{x}))$, then $X$ is a cause of\ \ $Y$. 
\end{definition}
The interventional distribution, $\mathbb{P}(Y|\mathrm{do}(x))$, is only attainable if before the experiment is conducted the experimenter has made sure $X=x$, i.e. the experimenter has \emph{intervened}.
If the above definition is satisfied, we denote this by $X\rightarrow Y$. It is immediately clear that the above definition can hold in both directions. Further, for the task at hand, to estimate the causal direction from $\mathbb{P}_{(X,Y)}$, without access to the interventional distribution apparent in the definition, one can only make qualified guesses. 

Imposing model assumptions can, in the spirit of Occam's razor, be seen as qualified
guessing. However, any such a priori interpretation of the data will bias the causal
prediction. To minimize such bias, we advocate a bivariate causal
inference approach that tries to stay clear of scores tied to probabilistic models,
and only rely on a test statistic that is well-defined for almost all datasets.




Like \citet{pearl}, we consider causal structures that are DAGs. Then, if
$X \rightarrow Y$, we must also have $X \rightarrow g(Y)$ for any function
$g$, since the contrary would construct a cycle. If $g$ is a bijection, this is equivalent to 
$f(X)\rightarrow Y$, where $f=g^{-1}$. This motivates our guiding principle.
\begin{principle}[Invariant causality] \label{prin:inv}
  A deterministic bijective reparametrization of the observed variables does not change the causal direction.
\end{principle}
We only consider bijections, as we \textit{a priori} do not
know if $X\rightarrow Y$ or $Y\rightarrow X$. The above then states that the causal 
relationship between $X$ and $Y$ is the same as between $f(X)$ and $g(Y)$, for bijections $f$ and $g$. 
Equivalently, our choice of units should not influence the causal direction;
i.e., the marginal distributions of $X$ and $Y$ must not matter.
Note that most model-based causal inference schemes are not closed under
nonlinear reparametrizations and, hence, violate Principle~\ref{prin:inv}. For instance,
a nonlinear reparametrization of an ANM does not yield another ANM.
\begin{figure}
    \centering
    \includegraphics{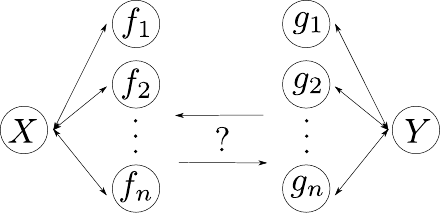}
    \caption{Visualisation of Principle \ref{prin:inv}. If $X$ is a cause of $Y$, then \emph{theoretically} arrows from $f_i$ to $g_i$ must point right for \emph{all} $i=1,\ldots,n$. We suggest to test this \emph{empirically}.}
    \label{fig:princ}
\end{figure}

Principle \ref{prin:inv} is illustrated in Figure \ref{fig:princ}. If for some $f_i$ and $g_j$, where $i,j=1,\ldots,n$, we have that $g_j$ is a cause of $f_i$, then $X$ is not a cause of $Y$. Likewise, if $f_i$ is a cause of $g_j$, then $Y$ can not be a cause of $X$, since $X$ there is a direct path from $X$ to $Y$ in the causal graph. Our idea is to construct $n$ bijections of both $X$ and $Y$, and test the causal relationship among these. If the decisions are unanimous, the causal link is likely strong. If they are inconsistent, this gives uncertainty in the causal estimator and we may interpret this inconsistency over bijections as uncertainty associated with the causal decision making. 

This discussion of invariances in causal estimators has not involved how to realize Principle \ref{prin:inv}. To this end, we consider the setup from \citet{bloebaum}. The inequality in Eq.~\ref{bloe} is shown to hold for small noise settings, when the condition
\begin{equation}
	\text{Cov}\left(\!\frac{\partial \mathbb{E}[Y|X\!=\!x]}{\partial x}, \mathbb{E}[\text{Var}(Y|X\!=\!x)]p_X(x)\!\right)\!=\!0,
\end{equation}
is satisfied. Here $p_X$ denotes the marginal distribution of the cause $X$. This criterion is similar to IGCI's idea that the expected log-derivative of the conditional mean is uncorrelated with the marginal distribution of the cause, and positively correlated in the anti-causal direction. These '\emph{uncorrelated mechanism}' ideas \cite{daniusis} fall under the causal principle of modularity and autonomy. For a broader review see \citet{elements}.  

In summary, \citet{bloebaum} prove that under similar conditions to what we shall impose, then the prediction error is greater in the anti-causal direction compared to the causal -- at least when the noise is small. Experimentally, they do regression by predetermined \emph{types}, such as polynomial or neural nets. We are interested in marginalising the underlying distribution, thus it is not obvious that some parametric form of regression should be robust to this. In the next section, we present a non-parametric estimator of the regression error. This should be seen as a means to realizing Principle~\ref{prin:inv}. If anything, causal inference is about decision-making under imperfect or uncertain information.

These are the outlines of the present work, which we use to derive
a simple causal inference scheme (Sec.~\ref{sec:qv}). We evaluate this
scheme in Sec.~\ref{experiments} and find that the empirical performance is on par
with current standard methodologies, but with the additional benefit that we
provide well-calibrated uncertainties over causal predictions.
On this path, we further derive and validate an extension to handle more than two variables
and find that this establishes a link between our proposed estimator and classic conditional independence tests for causal structures \cite{pearl}.
All proofs are in the supplementary materials.
\section{Quadratic Variation in Causal Discovery}\label{sec:qv}
If $f(X)$ is a predictor of $Y$, then $f(X)$ (trivially) correlates with $Y$. This motivates us to measure the correlation between $Y$ and the predictor $\mathbb{E}[Y|X]$. We will show that we can quantify this completely non-parametrically, i.e. not making distributional assumptions on $X$ and $Y$, besides finite second moments. In Sec.~\ref{sec:repara} we will show how this also allows us to apply Principle~\ref{prin:inv}.

To derive an estimator of this correlation, we first recap some theory from stochastic processes. Let $Y_t$ denote a real-valued stochastic process on some probability space, and $t>0$. The \emph{quadratic variation} \cite{durrett} of $Y_t$ is the increasing process defined as
\begin{equation}\label{qv}
\langle Y\rangle_t := \lim_{S\rightarrow 0}\sum_{i=1}^{n}(Y_{t_i}-Y_{t_{i-1}})^2,
\end{equation}
where $S$ is the mesh\footnote{For a partition $0<t_1<t_2<\ldots<t_N<t$, we denote the mesh as the longest distance between two points $\max\{(t_{i+1}-t_i)|i=1,\ldots,N-1)\}$.} of partitions of the interval $[0,t]$.  We define the \emph{mean quadratic variation (MQV)} as the scaling $\langle Y\rangle_t / t$, which can be seen as a measure of averaged noise over the time interval $[0,t]$. Notice that estimators akin to Eq.~\eqref{qv} for non-time series are well-known in non-parametric regression \cite{diffbased}. 

For the problem at hand, consider two real-valued random variables $X$ and $Y$ from a joint distribution $\mathbb{P}_{(X,Y)}$. Similar to other methods of causal discovery, we shall see $Y$ as a function of $X$, and vice versa. In particular we view it as a stochastic process on the interval supp$(X)$, which we assume to be bounded.
\begin{theorem}\label{maintheorem}
	Let $X$ have support on a compact and connected subset $C$ of $\mathbb{R}$,
	and assume that $\mathbb{E}[Y|X=x]$ is a continuous differentiable function
	over $C$. Assume $\mathbb{E}Y^2 <\infty$. Let further $(x_i,y_i)$, $i=1,\ldots,N$,
	be iid samples from $\mathbb{P}_{(X,Y)}$. If we order, such that $x_{i+1}\geq x_i$ for all $i=1,\ldots,N-1$, then it holds that
	\begin{align}
	\frac{1}{N-1}\sum_{i=1}^{N-1}(y_{i+1}-y_{i})^2\rightarrow 2\mathbb{E}\textrm{\emph{Var}}(Y|X),
	\end{align}
	as $N\rightarrow \infty$.
\end{theorem}
Theorem~\ref{maintheorem} motivates computing the following quantity for unit variance observations
\begin{equation}\label{estimator}
C_{X\rightarrow Y} := 1 - \frac{1}{2(N-1)}\sum_{i=1}^{N-1}\big(y_{i+1}-y_i\big)^2,
\end{equation}
since
\begingroup
\allowdisplaybreaks
\begin{align}
C_{X\rightarrow Y} &\rightarrow
1 - \frac{\mathbb{E}[\text{Var}(Y|X)]}{\text{Var}(Y)}
\\&= 1 - \frac{\text{Var}(Y) - \text{Var}(\mathbb{E}[Y|X])}{\text{Var}(Y)}
\\&= \frac{\text{Var}(\mathbb{E}[Y|X])}{\text{Var}(Y)} \\
&= \text{Corr}(\mathbb{E}[Y|X],Y)^2,
\end{align}
\endgroup
as $N \rightarrow \infty$. 
Eq.~\ref{estimator} measures the quality of a prediction of
$Y$ from $X$ without realizing the implied regression, and without making
specific assumptions over this regression, i.e. it measures the regression error non-parametrically.
A causal inference scheme, as suggested by Bloebaum's condition, is then to
infer the direction $X\rightarrow Y$, if $C_{X\rightarrow Y}\!>\!C_{Y\rightarrow X}$;
and symmetrically for the other direction. 

Notice the similarity here with the estimator in \cite{bloebaum} (also seen in Eq.~\eqref{bloe}), this imply that doing causal inference with \eqref{estimator}, inherits the guarantees formulated there. Notice the regression performed in \cite{bloebaum} is here implicitly done \emph{non-parametrically}, and as such with fewer structural assumptions.
For future reference, we denote the estimator \eqref{estimator} as the \emph{Mean Quadratic Variation (MQV)}. 

While we advocate a model-free approach, the above analysis does make \emph{some}
assumptions, which should be understood prior to drawing conclusions from data.
\begin{itemize}
	\item We assume the causal mechanism is \emph{continuous}.
	\item We assume $X$ has compact and connected support in order to bound
	the mesh of the partition generated by the sample. 
	\item We assume there exist a unique sorting of $x$, which is not the case if there are duplicated values. If duplicates are present, MQV can potentially --- by statistical anomaly --- sort the $y$ values and detect a signal which is not there.
\end{itemize}

We circumvent with the last two issues by resampling and perturbing the data.
Based on the given sample, we estimate the underlying probability distribution $\tilde{\mathbb{P}}_{(X,Y)}$, then resample the same amount of data points from this distribution and reevaluate MQV \eqref{estimator}. We repeat this procedure several times.
This approach both secures unique sorting and has an element of bootstrapping
that quantifies the sensitivity to unusual observations. 
This approach gives empirical distributions of both $C_{X\rightarrow Y}$ and $C_{Y\rightarrow X}$,
and we can then assert probabilities to the event $C_{X \rightarrow Y }>C_{Y\rightarrow X}$, and its mutual counterpart.
Algorithm~\ref{no-bijtest} summarize these ideas, and a practical realization is described in Sec.~\ref{experiments}.
\begin{algorithm}
	\caption{}\label{no-bijtest}
	\begin{algorithmic}[1]
		\State \textbf{Input} $N$ iid samples of $(X,Y)$.
		\State $\mu \gets $ Estimate the underlying probability distribution of $(X,Y)$.
		\For {$i$ from $1$ through $m$}
		\State $(\tilde{X},\tilde{Y}) \gets $ Sample $N$ points ($\tilde{x}$,$\tilde{y}$) from $\mu$.
		\State $Cx_i \gets C_{\tilde{X}\rightarrow\tilde{Y}}$; $Cy_i \gets C_{\tilde{Y}\rightarrow\tilde{X}}$ 
		\EndFor
		\State $p_x \gets \mu(Cx>Cy)$; $p_y \gets \mu(Cy>Cx)$
	\end{algorithmic}
\end{algorithm}

\subsection{Weak Identifiability}
The guarantees by \citet{bloebaum} apply to our approach too. We can, however, make some insights into when our approach is \emph{sensible}. First, we consider when it should \emph{not} be relied upon.
\begin{proposition}\label{prop:linear}
	Let $a,b,c,d\in\mathbb{R}$ and $a,c\neq0$. Assume $X$ and $Y$ are random variables with compact support. Then
	\begin{enumerate}
		\item[(1)] If we have that $\mathbb{E}[Y|X=x]= ax+ b$ and $\mathbb{E}[X|Y=y]=cy+d$, then $C_{X\rightarrow Y}=C_{Y\rightarrow X}$, in the limit of infinite data. 
		\item[(2)] We have $C_{aX+b \rightarrow cY + d}=C_{X\rightarrow Y}$.
	\end{enumerate}
\end{proposition}

This tells us that when the relationship between $X$ and $Y$ is near linear, we
cannot make an informed decision. Note that the use of bootstrapping 
ensure that both decisions have low confidence, such that the user is at least
aware of the lack of identifiability. Although, the linear case is often ideal when considering structural equation models, it is not necessarily simpler in general.

In the noise-free setting, more formal statements can be made.
\begin{proposition}\label{prop:nonoise}
	If $X$ and $Y$ are random variables with compact support, and there exists
	measurable $f$ such that $Y=f(X)$, then $C_{X\rightarrow Y}\geq C_{Y\rightarrow X}$,
	in the limit of infinite data.
\end{proposition}
This directly ties into the definition of causality, since if there is no noise
we have that $\mathbb{P}_{Y|X=x} = \mathbb{P}_{Y|\mathrm{do}(x)}$, such that
the correct causal decision will be taken.
Practically, this indicates, that we should make few incorrect decision in the
low-noise regime.
Notice that if $f$ is bijective, there exists a function $g=f^{-1}$ such that
$X=g(Y)$. Then $C_{X\rightarrow Y} = C_{Y\rightarrow X}$,
such that any taken decision will have low confidence. As before, bootstrapping
implies that the user is \emph{aware of this low confidence}.

%

\subsection{Reparametrization Invariance}\label{sec:repara}
\begin{figure}
	\centering
	\hfill\begin{minipage}{0.40\textwidth}
		\centering
		\includegraphics[width=\textwidth]{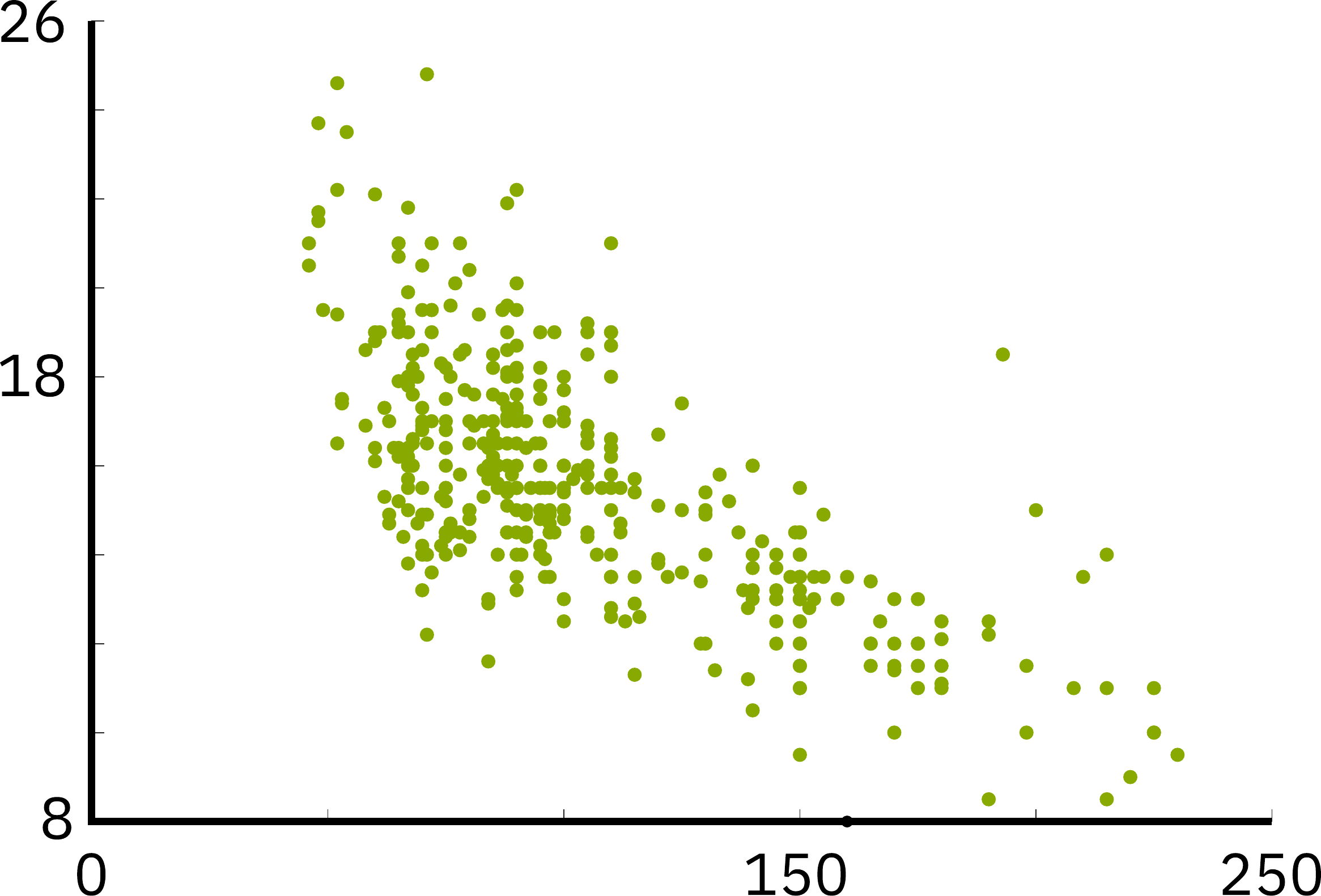} 
	\end{minipage}\hfill
	\hfill\begin{minipage}{0.40\textwidth}
		\centering
		\includegraphics[width=\textwidth]{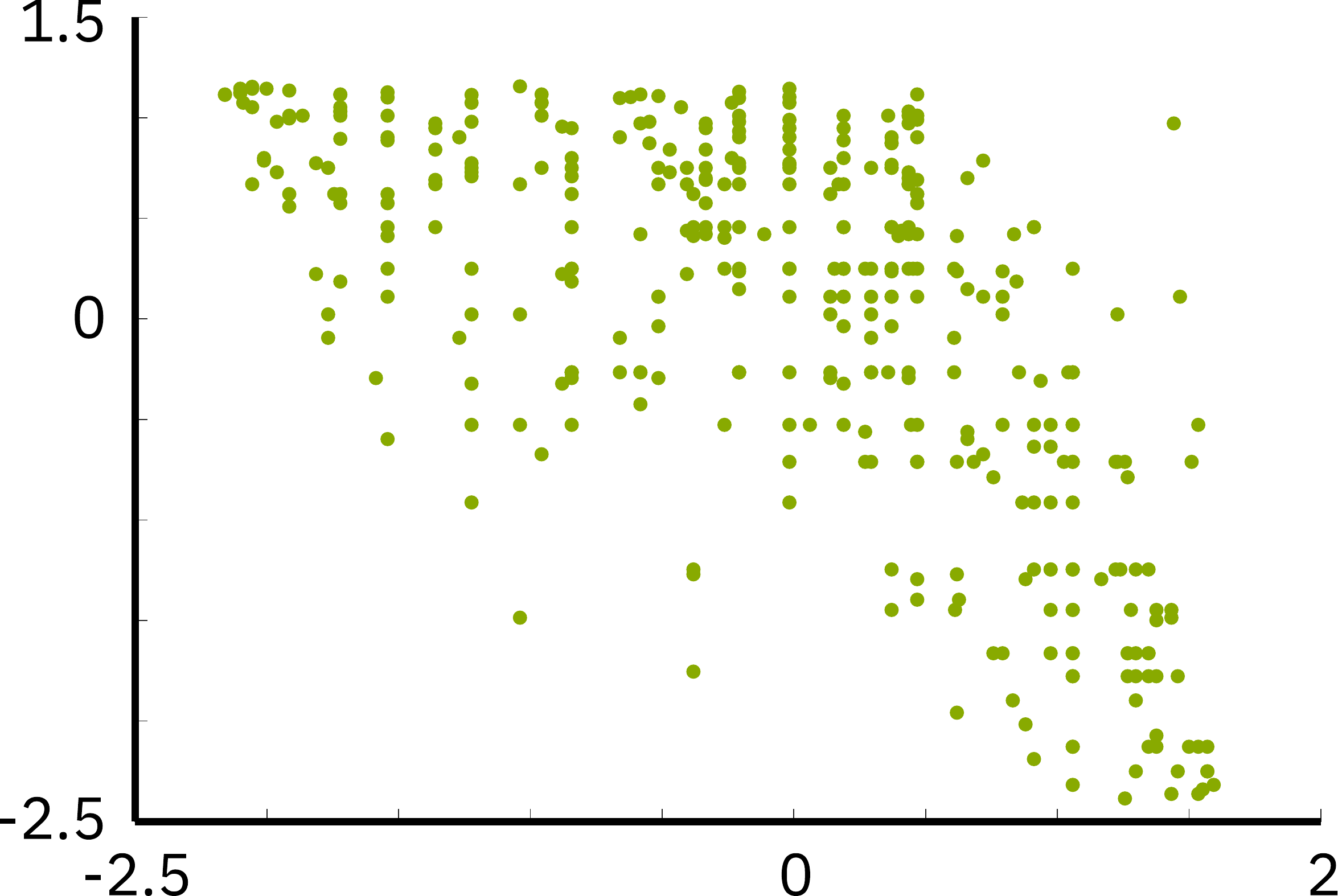}
	\end{minipage}\hfill
	\caption{\emph{Left:} a scatterplot of Horsepower $(x)$ vs. Acceleration $(y)$ from the Auto-MPG Dataset (pair0016 from CEP \cite{benchmark}).
		\emph{Right:} a bijective reparametrizations of the same. The invariance principle states, that the causal decisions taken, should be identical for these two datasets.}\label{repara}
\end{figure}
Principle~\ref{prin:inv} informs us that causal decisions should not rely on a specific parametrization of the observations; it is an intrinsic property of the system, rather than a property of the observation space in which we measure. One approach to extracting this intrinsic property is to consider a large number of different parametrizations, in order to be partially invariant to the particular choice of parametrization. Since MQV \eqref{estimator} itself is non-parametric, it is meaningful to evaluate $C_{X \rightarrow Y}$ under different parametrizations of the observed data. The simplicity of MQV \eqref{estimator}, thus, allow us to realize Principle~\ref{prin:inv}. We emphasize that this principle is truly causal, yet most model-based approaches, ANMs in particular, cannot aid its realization. Analogously, parametric approaches to regression-error based causal inference \cite{bloebaum} are not, in general, invariant to changes in marginal distributions.

We provide a straight-forward realization of the above
considerations: define a distribution over bijective reparametrizations $f, g$,
sample $C_{f(X) \rightarrow g(Y)}$, and infer a causal direction.
We postpone practical implementation details to Sec.~\ref{experiments} and supplementary materials.
%
\begin{algorithm}
	\caption{}\label{bij_test}
	\begin{algorithmic}[1]
		\State \textbf{Input} $N$ iid samples of $(X,Y)$. Positive integers $M$ and $m$.
		\For {$j$ from $1$ through $M$}
		\State Generate random bijection $f$ and $g$.	
		\State $\mu \gets $ Estimate the underlying probability distribution of $(f(X),g(Y))$.
		\For {$i$ from $1$ through $m$}
		\State $(\tilde{F},\tilde{G}) \gets $ Sample $N$ points $(\tilde{f}$,$\tilde{g})$ from $\mu$.
		\State $Cx_{ij} \gets C_{\tilde{F}\rightarrow\tilde{G}}$; $Cy_{ij} \gets C_{\tilde{G}\rightarrow\tilde{F}}$
		\EndFor
		\EndFor
		\State From samples $Cx_{ij}$ and $Cy_{ij}$ empirically evaluate $p_x = \mathbb{P}(C_X > C_Y)$ and $p_y= 1-p_x$.
	\end{algorithmic}
\end{algorithm}


In Fig.~\ref{repara} we illustrate the invariance principle. Naturally, the marginal distributions of $X$ and $Y$ changes dramatically, and we may think of Algorithm 2 as \emph{integrating out} the marginals. Principle~\ref{prin:inv} dictates that the causal link between $X$ and $Y$ is unaltered under these changes; our method then investigates if the estimator (Eq. \ref{estimator}) is too. If this is not the case, we may choose to say that our method cannot estimate a causal relationship. It is clear that any causal inference method based on distributional aspects of the observed is sensitive to these bijections. We empirically investigate this in Sec.~\ref{robustness}.

\subsection{Causal Confidence}
The proposed approach can be realized through sampling.
This imply that our approach naturally assigns probabilities $p_x$ and $p_y$
to each causal direction. From this, we can near-trivially define a \emph{confidence}, which allow us to rank decisions, as
\begin{equation}\label{rankheuristic}
\text{conf}(d):=|p_x(d)-0.5|.
\end{equation}
It is a feature of our approach, that the confidence in a decision is an \emph{integral} part of the decision itself. Notice $p_x$ reflect both statistical and model uncertainties, respectively thinking of the bootstrap and reparametrization considerations. 

\subsection{The Multivariate Generalization}
The main focus of this paper is the bivariate case, but the idea neatly generalize to the multivariate setting.
So far, we have looked at the variance process (more specifically the \emph{MQV}), but by the polarization identity \cite{durrett}, we can expand to triplets $(X,Y,Z)$ and see that the covariance conditioned on $Z$ is
\begin{equation}\label{covar}
\text{Cov}(X,Y)_Z \!:=\! \sum_{i=1}^{N-1}\!w_{i,i+1}\Bigg(\!\big(s_{i+1}-s_i\big)^2-(t_{i+1}-t_i)^2\!\Bigg),
\end{equation}
where $s_i=x_i+y_i$ and $t_i=x_i-y_i$ and $\sum_{i=1}^{N-1}w_{i,i+1}=\frac{1}{8}$.
Furthermore, the sorting is chosen such that $z_i\leq z_{i+1}$. This expression
is symmetric in $X$ and $Y$, but not in $Z$; and notice how Eq.~\ref{covar} in
its unaveraged version is exactly the covariance process from stochastic process
theory. Hence, we call it the \emph{mean co-quadratic variation}. We state the following Theorem without proof here, as it is analogous to Theorem~\ref{maintheorem}.
\begin{theorem}\label{covarthm}
	Let $(x_i,y_i,z_i)_{i=1,\ldots,N}$ be iid samples from $\mathbb{P}_{(X,Y,Z)}$, and assume that $Z$ has compact and connected support $C\subset\mathbb{R}$. Assume further that $\mathbb{E}[X|Z=z]$ and $\mathbb{E}[Y|Z=z]$ are both continuously differentiable over $C$. Define $s_i=x_i+y_i$ and $t_i=x_i-y_i$ for all $i=1,\ldots,N$. Then
	\begin{equation}
	\frac{1}{8(N-1)}\sum_{i=1}^{N-1}\Bigg(\big(s_{i+1}-s_i\big)^2-(t_{i+1}-t_i)^2\Bigg),
	\end{equation}
	tends to $\mathbb{E}\text{\emph{Cov}}(X,Y|Z)$ as $N\rightarrow\infty$.
\end{theorem}
By the law of total covariance, we have
\begin{equation}
\text{Cov}(X,Y)\!-\!\text{Cov}(\mathbb{E}[X|Z],\mathbb{E}[Y|Z])\!=\!\mathbb{E}[\text{Cov}(X,Y|Z)],
\end{equation}
implying that if Eq.~\ref{covar} is close to zero, then most of the covariation
between $X$ and $Y$ can be explained by $Z$. This indicates that $X$ and $Y$ might
be independent given $Z$. Of course, this is generally not a sufficient condition,
but it is necessary. The following statement gives sufficient conditions \cite{probess}.
\begin{theorem}\label{covarsuff}
	Two random variables $X$ and $Y$ are independent if and only if
	\begin{equation}
	\text{Cov}(f(X),g(Y))=0,
	\end{equation}
	for any pair of functions $f$ and $g$ that are bounded and continuous.
\end{theorem}

This sufficient condition allows for a simple conditional independence test:
transform the observed values with bounded continuous functions and check if
Eq.~\ref{covar} is zero for any such transformation. This naive test is exactly how we algorithmically realize Principle~\ref{prin:inv}, which illustrates that
our non-parametric estimators aligns with the fundamental ideas from \citet{pearl} and DAG estimation. 

\section{Experiments}\label{experiments}
We now evaluate the empirical behavior of the proposed model-free approach.
We consider ANM and IGCI as baseline methods, and report results first
on simulated data, and then on the real-world CEP benchmark dataset \cite{benchmark}.
For all comparisons below, when we state IGCI, we mean the slope-based estimator with uniform reference measure. For ANM, we applied the GP regression and the Hilbert-Schmidt Independence Criterion \cite{hsic}. Implementations are from the publicly available code given by \citet{benchmark}.
We rank our decisions based on the confidence score in Eq.~\ref{rankheuristic},
while ANM and IGCI come with their own confidence scores \cite{benchmark}.
Algorithmic details of our estimator are available in the supplementary material
alongside the associated source code.

We shall also compare to \emph{Regression Error-based Causal Inference} (RECI) from \citet{bloebaum}, as their approach is highly similar to the one presented here. This imply we have to choose a method of regression, and we regress by using the logistic-function class. This decision was made based on what we found were the overall best performance in their paper.

It was a motivation for us to present a non-parametric way of doing (implicit) regression, that would alleviate the need to pick a fair regression method for both causal and anti-causal direction. Further, the non-parametricity in our estimator is essential to apply bijections meaningfully.

Lastly, we performed a small experiment on the sensitivity of how the reparametrizations were sampled. The outline of this experiments was that as long as the bijections were \emph{diverse} enough, there is little sensitivity to the choice of distribution. 
\subsection{Simulated Pairs}\label{simres}
The data considered here is $100$ pairs, each consisting of $1000$ observations, simulated according to the procedure introduced by \citet{benchmark}; trying to mimic real-world data. There are four setups: the general simulated data (SIM), the data generated with low noise to the effect (SIM-ln), the data with one confounder present (SIM-c), and finally the data where the cause is Gaussian and the additive noise is too (SIM-G). 

As we have observed, we would expect our method to perform well at least on the low-noise data, as one would too for IGCI. The results for all datasets are visualized in Fig.~\ref{simpair}. This figure should be read from right to left as taking all decisions, we then sequentially discard the decisions we are most uncertain about. The 10 blue lines are outputs from Algorithm 2, indicating the inherent randomness in the decision-making. We see that our method outperforms IGCI in most cases, and is comparable to ANM. Equivalently, the cyan lines are the outputs from Algorithm 1 - that is, without bijections.

In all experiments, we note that our choice of ranking \eqref{rankheuristic} 
prefers easier decisions, which is evident from the \emph{concave} shape of the result curve. Note that uncertainty is larger for decisions that are considered difficult (low confidence), but the performance on high confidence decisions is generally better than both ANM and IGCI. From Fig.~\ref{simpair} it is visible that MQV reports: `\emph{I don't know}' when one blue line turns into many. At this crucial point MQV is consistently as good or better than ANM. 

Interestingly, if we take decisions with ANM in the same order as MQV, we obtain more preferable concave curves; in fact such that the performance resembles MQV on high confidence situations. We investigated this on all 4 datasets, and the concavity is visualized in Fig.~\ref{conc}. Here all decision are taken as determined by ANM, but ordered wrt. Eq.~\ref{rankheuristic}; the lines indicate the difference to the black lines of Fig.~\ref{simpair}. Hence the lines are bound to go through $(0,0)$ and $(100,0)$, and in between any positive number imply an improvement over ANM's own ranking.
In particular, we note that for decisions where our estimator is certain, we generally improve upon ANM's ranking.

The overall performance of RECI is shown in Table \ref{RECI-sim}. We see that, unsurprisingly, this is \emph{very} similar to our approach without bijections. Another key observation from Figure \ref{simpair} is that on the most 'nature-like' datasets, SIM and SIM-c, taking observations into account improves overall decision making. On SIM-G we conjecture there might be a bias in our estimator (without bijections), why it therefore might still be more 'safe' to include bijections. This conjecture is based on that the Gaussian is the maximum entropy distribution when mean and variance is known (standardized in our case), which might make the regression error of $Y\rightarrow$ X tend to be larger if $X$ is Gaussian, than if it was not. This sort of bias is alleviated by marginalizing marginal distributions (with bijections).
\begin{table}
\centering
\begin{tabular}{lcccc}
\hline
                     & SIM                    & SIM-c                  & SIM-ln                 & SIM-G                  \\ \hline
RECI                 & $64.1\pm 1.4$          & $63.2\pm 2.4$          & $\mathbf{83.1\pm 2.3}$ & $\mathbf{74.9\pm 3.9}$ \\
MQV (w/o bijections) & $62.2\pm 0.9$          & $63.4\pm 0.1$          & $\mathbf{82.7\pm 0.8}$ & $\mathbf{73.5\pm 0.8}$ \\
MGV (w/ bijections)  & $\mathbf{68.8\pm 2.1}$ & $\mathbf{65.9\pm 2.5}$ & $\mathbf{82.0\pm 3.5}$ & $61.5\pm 2.6$          \\ \hline
\end{tabular}
\caption{Average number of correct decisions and standard deviations over 10 runs. Bold marks statistically significantly best method. Only regression-error based methods are listed, as ANM and IGCI do not have error-bars; their performance can be read from Figure \ref{simpair}.}
\label{RECI-sim}
\end{table}
\begin{figure}
	\centering
	\hfill\begin{minipage}{0.5\textwidth}
		\centering
		\includegraphics[width=\textwidth]{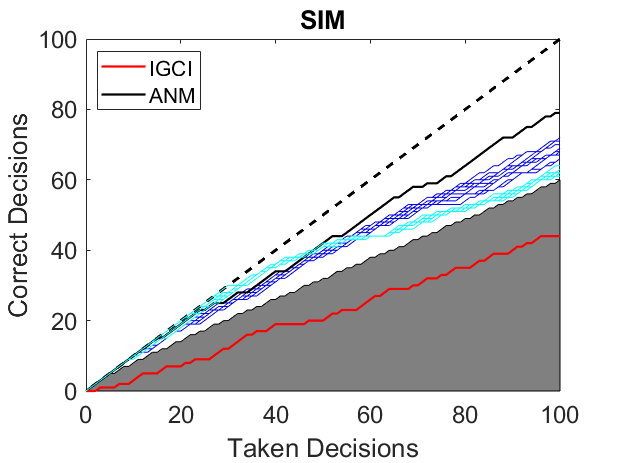} 
	\end{minipage}\hfill
	\hfill\begin{minipage}{0.5\textwidth}
		\centering
		\includegraphics[width=\textwidth]{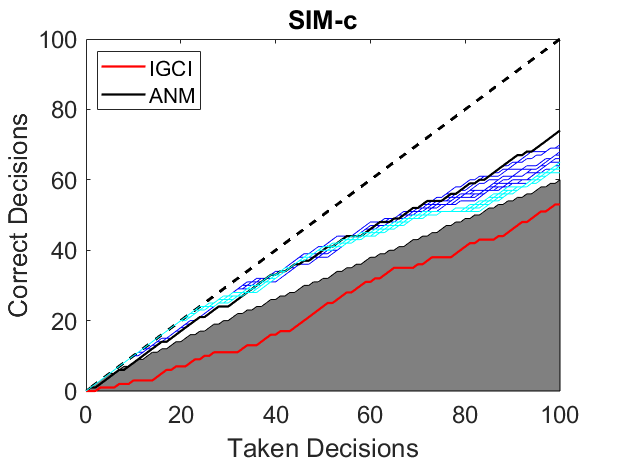} 
	\end{minipage}\hfill
	\hfill\begin{minipage}{0.5\textwidth}
		\centering
		\includegraphics[width=\textwidth]{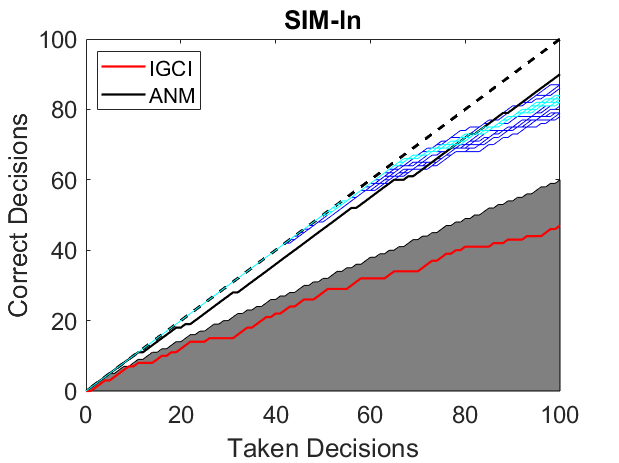} 
	\end{minipage}\hfill
	\hfill\begin{minipage}{0.5\textwidth}
		\centering
		\includegraphics[width=\textwidth]{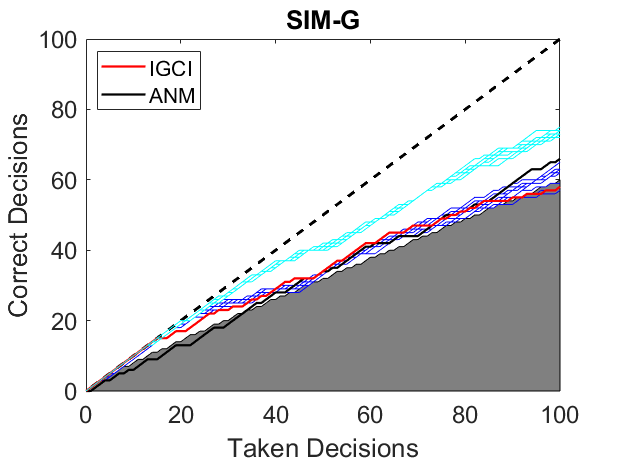} 
	\end{minipage}\hfill
	\caption{For the four different synthetic datasets, the blue lines (MQV with bijections) are the proportion of correct decisions where the decisions have been ranked according to the heuristic (\ref{rankheuristic}). The cyan lines are MQV without bijections. ANM and IGCI have other confidence scores \cite{benchmark}. The shaded area  is what falls below the 0.975 quantile of a binomial distribution with $p=0.5$.}\label{simpair}
\end{figure}
\begin{figure}
	\centering
	\includegraphics[width=\textwidth]{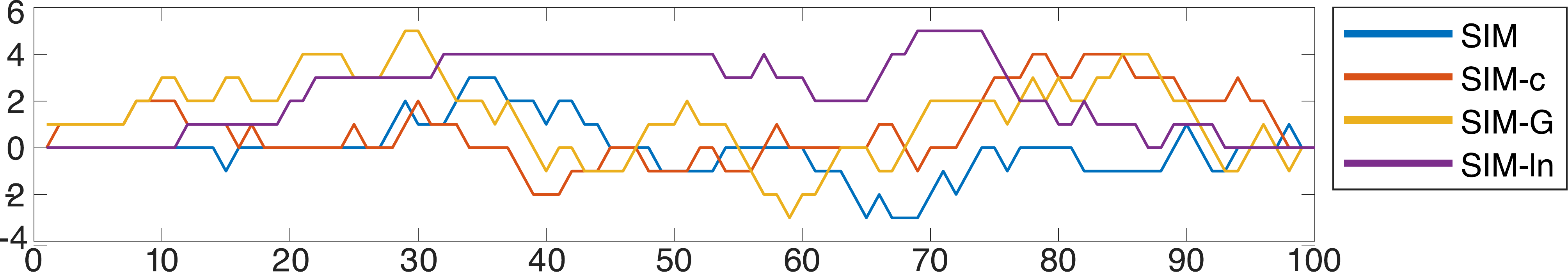} 
	\caption{Illustration of changes in concavity for ANM under changes in confidence score. Any curve above the constant line 0 imply more concave decision curve, hence the ranking we propose is better than the original \cite{benchmark}. Especially, concavity in the left-most side of the plot is important, as this reflect the most confident decisions.}\label{conc}
\end{figure}

\subsection{Real-World Data}
The \textsc{CauseEffectPairs (CEP)} database\footnote{https://webdav.tuebingen.mpg.de/cause-effect/ as it appeared in December 2019.} is currently $108$ datasets, of which $103$ are bivariate. It consists of real-world observations annotated with a causal direction \cite{benchmark}. As such, the $103$ pairs are not independent, as several originate from the same datasets, and to make up for this each pair has an associated weight. Our results on this dataset are plotted in Fig.~\ref{cep} and we see that we are comparable to other known methods when we \emph{integrate out} random bijections. The blue dots in the figure are our results for respectively Algorithm~\ref{no-bijtest} and \ref{bij_test} for 10 runs. Most runs for Algorithm~2 yielded accuracies in the range $0.62-0.64$, with one run having accuracy $0.66$ and two around $0.59$.
We see that Algorithm~2 is comparable to other known methods, while Algorithm~1
is subpar ($\sfrac{7}{10}$ runs had accuracy in $0.58-0.61$).
ANM yields an accuracy of $0.63$, and for IGCI $0.64$. Most importantly, this illustrates that Principle~\ref{prin:inv} is not hollow talk, since marginalizing bijections seem to significantly improve performance.

Over 10 runs the RECI method provide accuracies on the range of $[ 0.46, 0.62]$, averaging at $0.53$. This is worse that our approach, even without bijections, and not significantly better than random guessing.
\subsection{The Multivariate Generalization}
\begin{figure}
  \centering
	\includegraphics[width=0.37\textwidth]{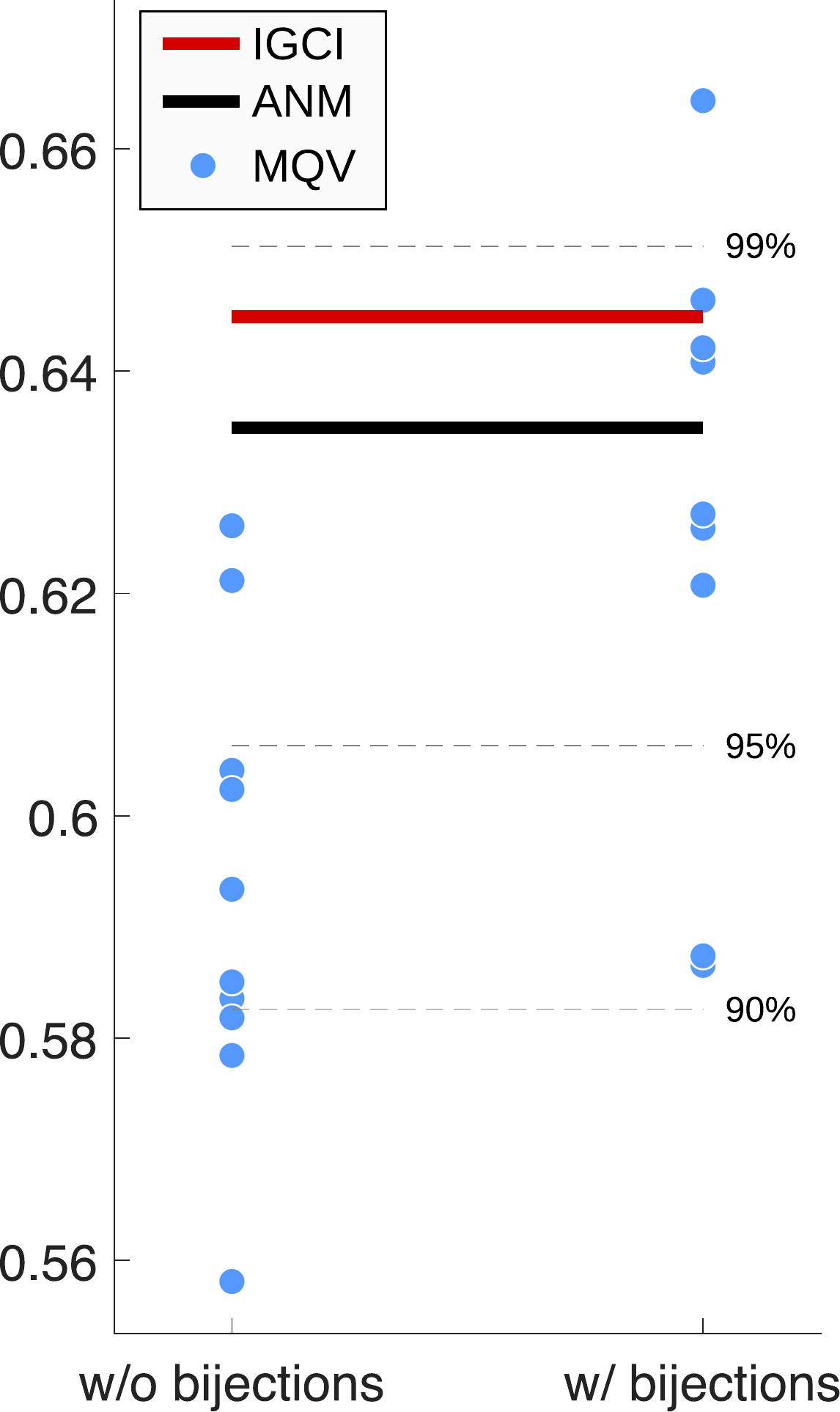}
	\caption{Performance on the \textsc{CEP}-benchmark. The blue dots (MQV) illustrate the inherent randomness in the algorithms. The dashed grey lines are quantiles had we tossed a fair coin for each decision. The average performance of RECI is $0.53$; below the plotted window.}\label{cep}
\end{figure}

\begin{figure}[t]
	\centering
	\vspace{-30mm}
	\begin{minipage}{0.5\textwidth}
		\centering
		\includegraphics[width=\textwidth]{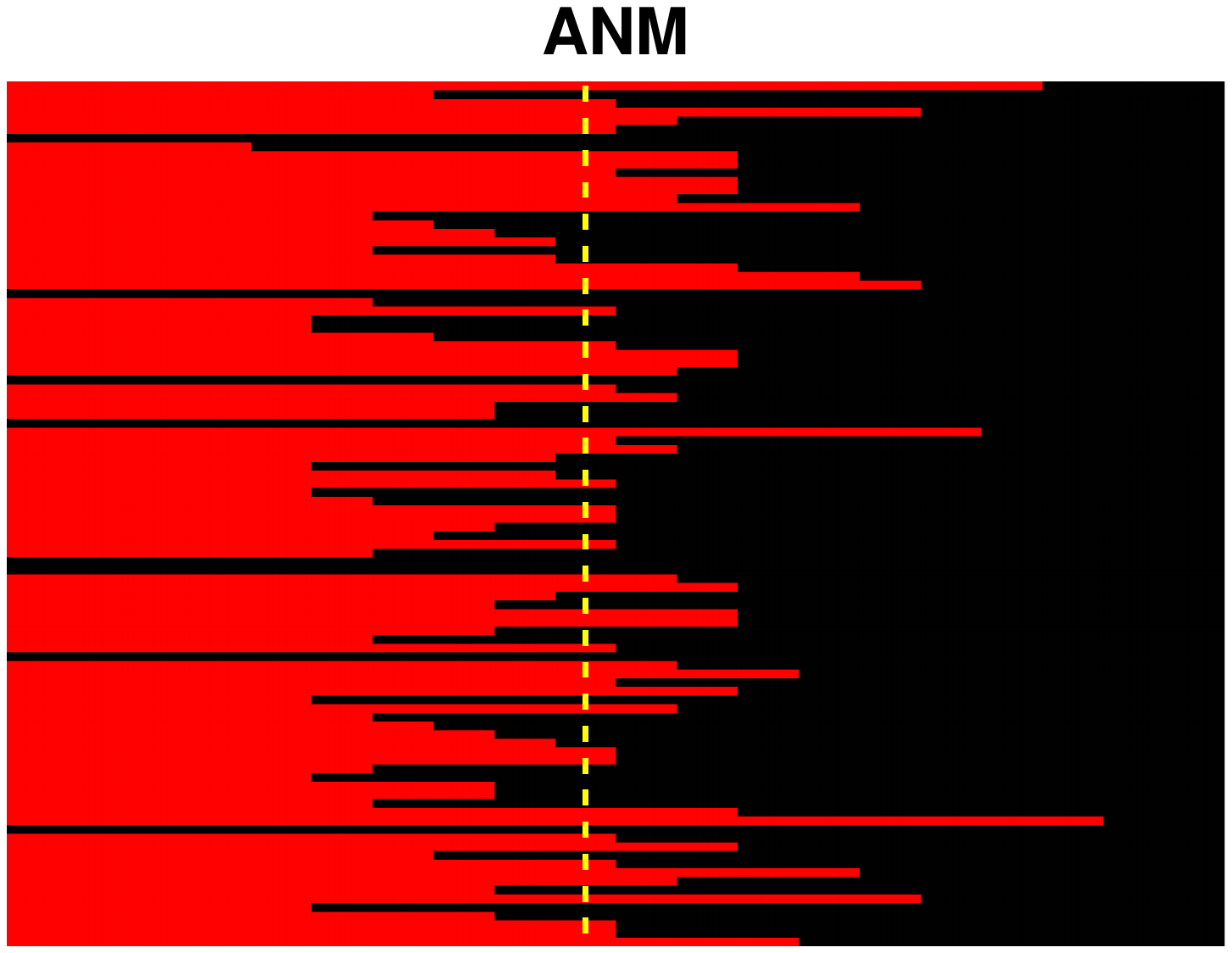} 
	\end{minipage}
	\begin{minipage}{0.5\textwidth}
		\centering
		\includegraphics[width=\textwidth]{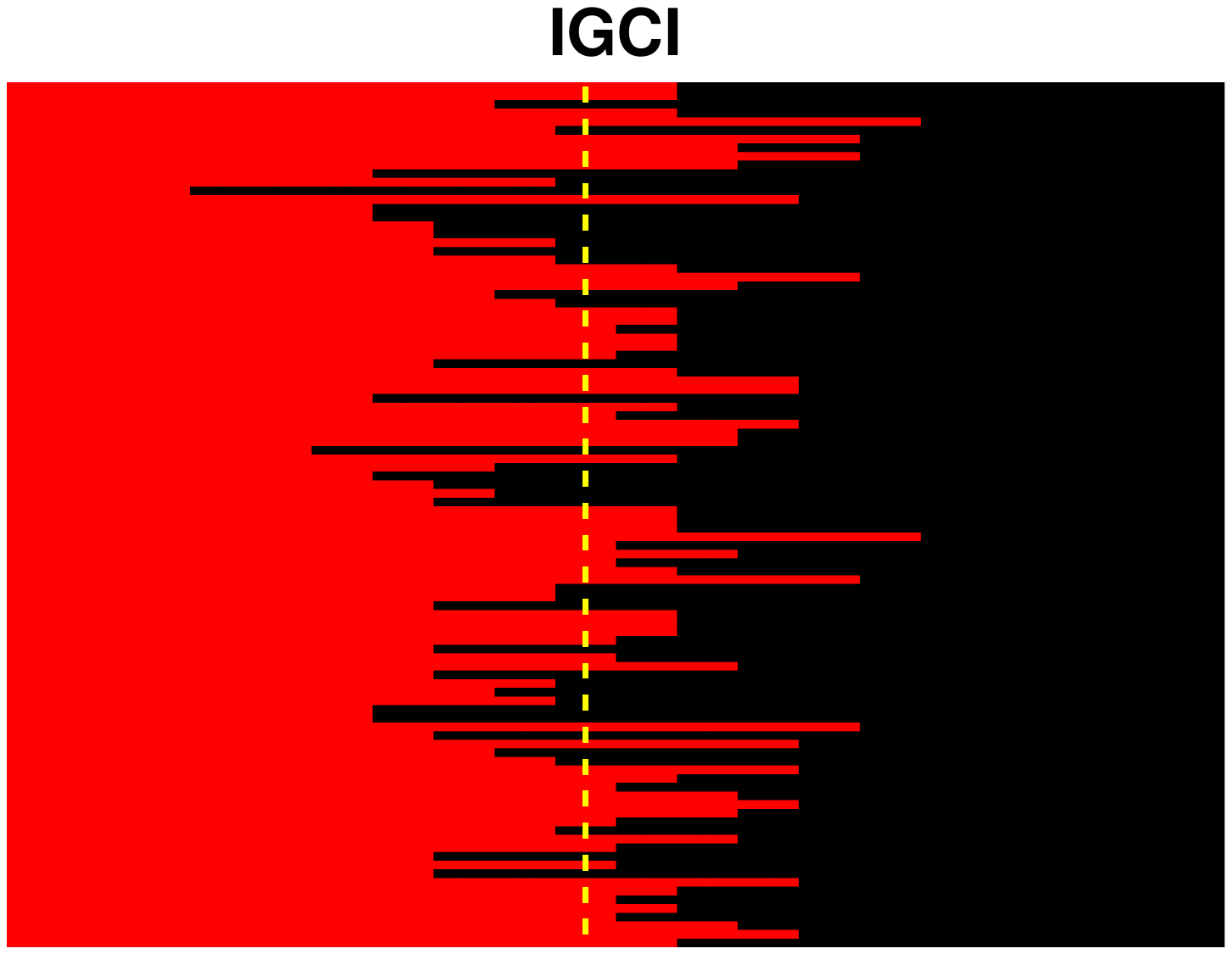} 
	\end{minipage}\hfill\vspace{-43mm}
	\begin{minipage}{0.5\textwidth}
		\centering
		\includegraphics[width=\textwidth]{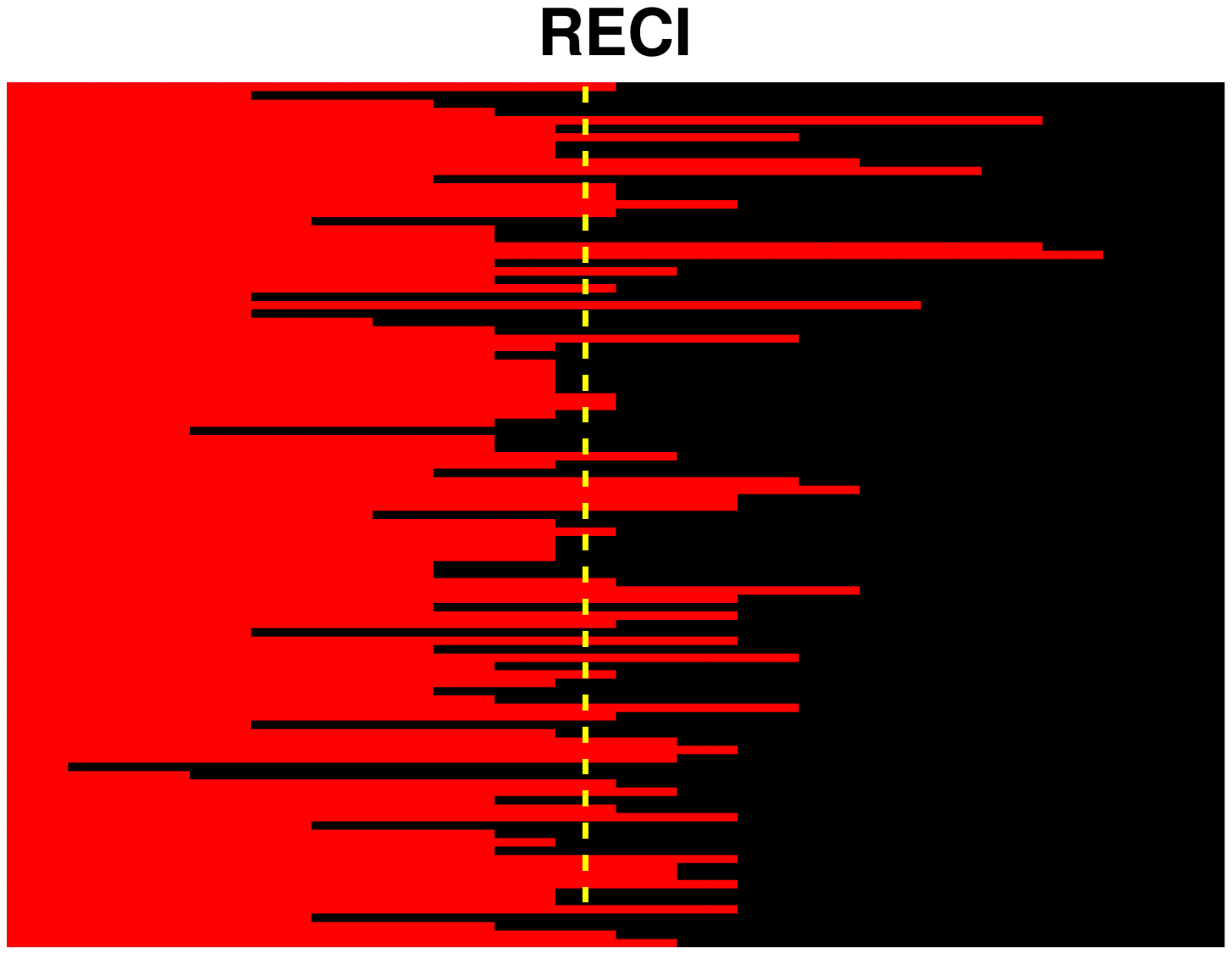} 
	\end{minipage}
	\begin{minipage}{0.5\textwidth}
		\centering
		\includegraphics[width=\textwidth]{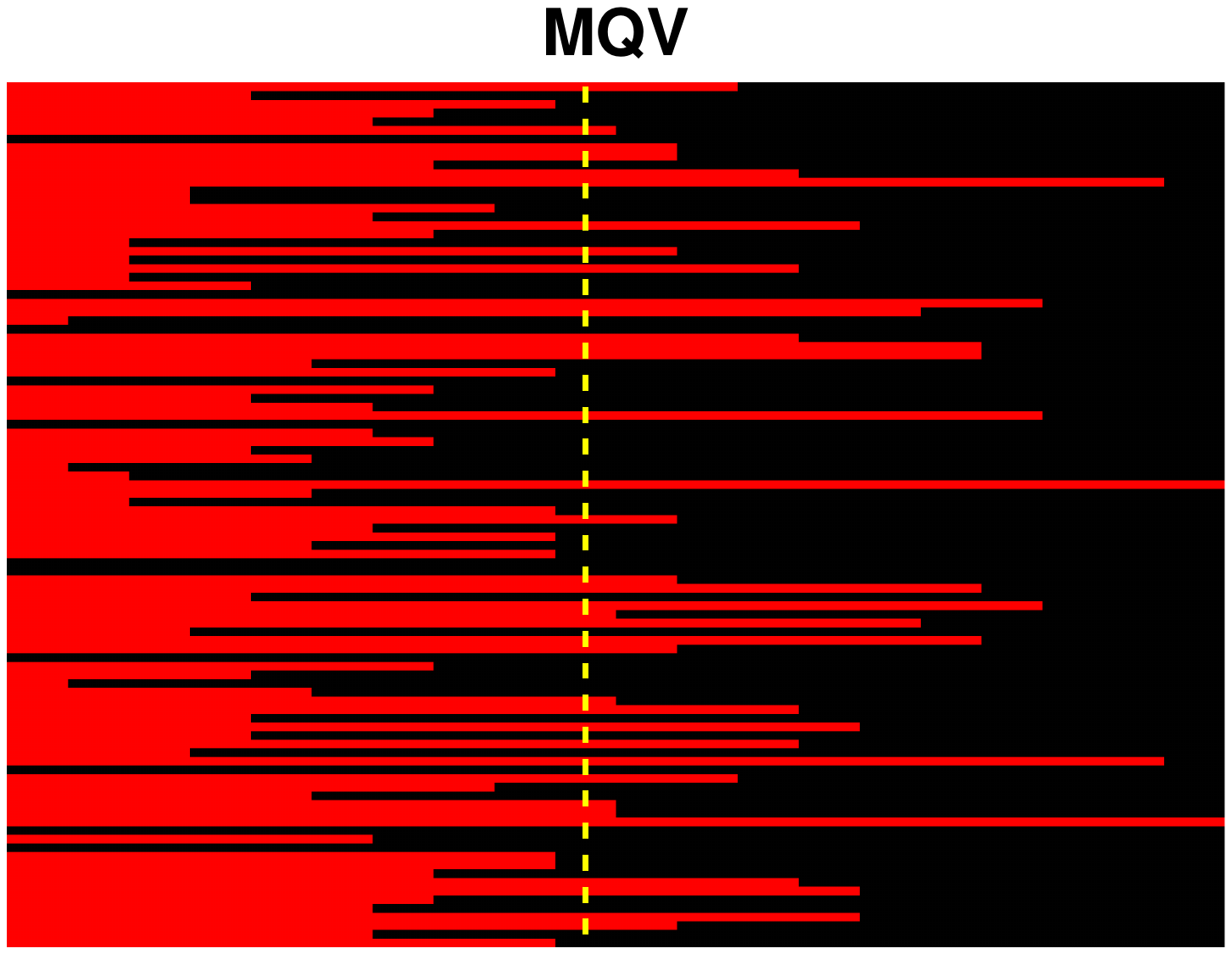} 
	\end{minipage}
	\vspace{-29mm}
	\caption{On the 100 pairs (rows) from SIM, we applied 20 random bijections (columns). Above illustrates how the bijections influenced the decision. Red is an incorrect decision. Table~\ref{entrop} quantitatively summarize the plots of this figure. We observe that MQV has 'fuller' bars, indicating that decisions are less influenced by bijections.}\label{sensitivity}
\end{figure}

We empirically illustrate the generalization to triplets $(X,Y,Z)$ on data generated similar to that in Sec.~\ref{simres}. We generate 100 DAGs of the type $X\leftarrow Z\rightarrow Y$, and a further 100 DAGs similar but with an added edge, either $X\rightarrow Y$ or $Y\rightarrow X$. For clarity, this means that we should find 100 times that $X\independent Y|Z$ and 100 times not. We construct a test that transform the variables $X$ and $Y$ with bijections, and test whether any such transformation make their absolute conditional covariance $\text{Cov}(f(X),g(Y))_Z$ (calculated as Eq. \eqref{covar}) exceed some threshold. In the following we have set this threshold to $0.15$, and rejected the hypothesis of independence if more than 1\% of the samples go above this.

The results in Table~\ref{condind} exemplify that (co-)quadratic variation is not misplaced in the causal framework, knowing that causality and conditional independence testing have been closely related for decades \cite{pearl}. We further notice that Theorem~\ref{covarthm} assumes $Z$ to be one-dimensional, but this extend to higher dimensions if one just finds a \emph{sorting} in this space. Keep in mind that, if the mesh tends to zero, then the convergence from above is still assured. For practical considerations one would then find a permutation of $z_i$, $i=1,\ldots,N$ such that $\max_{i=1,\ldots,N-1}\|z_{i+1}-z_i\|_2$ is as small as possible. This is a non-trivial problem for higher dimensions than 1. Suggestions here could be to use some kernel methods \cite{onvarest}, or some ranking on data manifolds \cite{rankonmanifolds}.
\begin{table}
	\centering
	\begin{tabular}{l|ll|l}
		& Positives & Negatives & Total \\ \hline
		True  & $99$         & $1$           & $100$     \\
		False & $3$          & $97$          & $100$    
	\end{tabular}
	\caption{Conditional independence results on synthetic data.}
	\label{condind}
\end{table}

\subsection{Robustness}\label{robustness}
We evaluate how restrictive model assumptions are for the invariance principle; more specifically we measure how robust the different causal inference methods are if we transform $X$ and $Y$ with bijections. This gives an indication of which methods align with Principle~\ref{prin:inv}. Fig.~\ref{sensitivity} has in each row one pair from the dataset SIM. To each of these pairs, we applied 20 random bijections and kept track of the decisions made. Black and red are respectively incorrect and correct decisions. 
We can see that MQV is more robust to Principle~\ref{prin:inv}, as we have more \emph{full bars} (or near full) along the rows, implying the decisions are less likely to be altered by the bijections. In fairness it should be stated that the bijections are not identical in-between plots. We may quantify this sensitivity with the entropy, i.e. for each pair evaluate $-d_1\log d_1 - d_2\log d_2$, where $d_1$ and $d_2$ are the fraction of times decision $X\rightarrow Y$ and $Y\rightarrow X$ were made. In Table~\ref{entrop} the average entropy of causal decisions over all pairs in a data set is listed, which indicates how robust a method is to bijections: small entropy imply robustness. 

Table~\ref{entrop} provide evidence to the hypothesis that assuming a model is not robust under random bijections. Our method deals better with this.
\begin{table}
\centering
		\begin{tabular}{l|lllll}
     & SIM               & SIM-c             & SIM-ln            & SIM-G             & CEP               \\ \hline
ANM  & $0.5953$          & $0.5838$          & $0.5457$          & $0.5748$          & $0.5591$          \\
IGCI & $0.6620$          & $0.6718$          & $0.6692$          & $0.6655$          & $\textbf{0.1247}$ \\
RECI & $0.6097$          & $0.6092$          & $0.5902$          & $0.6066$          & $0.6416$          \\
MQV  & $\textbf{0.4895}$ & $\textbf{0.4461}$ & $\textbf{0.4906}$ & $\textbf{0.4376}$ & $0.4381$         
\end{tabular}
	\caption{The mean entropy of decisions under random bijections.}
	\label{entrop}
\end{table}
One naturally also observes a clear deviant in Table~\ref{entrop}, IGCI-decisions on \textsc{CEP} are nearly closed under bijections, and there must be some entity in the data explaining this. Following up on this, we introduce a \textit{strawman estimator}
\begin{equation}
S_{X\rightarrow Y}:= \frac{\# \text{ of unique values in }X}{\# \text{ of unique values in }Y},
\end{equation}
and infer $X\rightarrow Y$ if $S_{X\rightarrow Y}<1$. Evidently this measure is invariant if we biject $X$ and $Y$, but its relation to causal decision taking is not evident. On \textsc{CEP} this procedure takes the same decision as IGCI on $98$ out of $103$ pairs, and the strawman estimator alone has an accuracy around $0.57\!-\!0.61$ (in 3 cases $S_{X\rightarrow Y}=S_{Y\rightarrow X}$, and we flip a coin). Thus, we conjecture that the success IGCI has had on the \textsc{CEP}-Benchmark is a spurious correlation due to duplicated values in the data. This is supported by the fact that IGCI discard duplicated values.\looseness=-1 

\section{Discussion and conclusion}

We took a novel approach to  bivariate causal discovery, by imposing invariance on the causal domain rather than distributional assumptions. We did this by quantifying the regression errors in a non-parametric manner, which allowed for us to meaningfully take advantage of the proposed invariant principle (Principle~\ref{prin:inv}). We provide a thorough empirical analysis on the impact of this principle.

The results show that this approach is feasible and is competitive with the current methodologies, that impose structural model assumptions. We find both the theoretical and computational ease of the approach highly appealing. However, we do not consider the present work complete, and we hope that future work in the field will take into account that if causal models are closed under reparametrizations, then so should its estimators. The results show that the non-parametricity of the mean quadratic variation (MQV) is more robust under reparametrizations, and that taking this into account significantly improves performance. This insight also proposed an explanation for the good performance of IGCI \cite{igci} on real-world observations to which there has been previous speculation \cite{benchmark, igci}. 

Further, we have demonstrated that MQV extends to higher dimensions in ways that are similar to the traditional conditional independence tests used for estimating DAGs \cite{pearl}. 

Finally, the presented method pays high attention to the uncertainty of any causal estimation, which results in a confidence measure that outperforms the baselines and shows good promise of detecting when it seems feasible to do causal inference with purely observational data; a query which is much more fundamental than the inference itself.

\subsection*{Acknowledgements}
This project has received funding from the European Research Council (ERC)
under the European Union’s Horizon 2020 research and innovation programme (grant agreement no 757360). MJ and SH were supported in part by a research grant (15334) from VILLUM FONDEN.

\setcounter{section}{0}
\renewcommand{\thesection}{\Alph{section}}

	\section{Experimental details}
	
	For the experimental setup in the paper, we here give the explicit and detailed description. See Algorithm \ref{no-bijtest} and \ref{bij_test} for notational help.
	
	For our own method, we consistently use $m\!=\!300$ and $M\!=\!100$ (see Algorithm~\ref{bij_test}),
	meaning we generate $100$ random bijections for each pair, and for each of these
	we subsample $300$ times, c.f.\ Sec. 2.
	We estimate the underlying probability distribution by Gaussian kernel density
	estimation, with Silverman's rule of thumb for bandwidth \cite{silverman}. This
	is a crude estimator for many pairs, but we leave it to future work to optimize
	this procedure of the algorithm; and to fairly compare on all pairs we choose it throughout.
	
	From a practical perspective we note that by Proposition~\ref{prop:linear} we may restrict this search to strictly increasing functions.
	
	\textbf{Generation random increasing functions} was done with the following setup: draw $\gamma$ from an inverse Gamma distribution with both shape and scale parameters set to 5. Generate a Gaussian Process (GP) $f$ with zero mean and covariance function $k(x,x')=\exp(-\frac{1}{2\gamma}\|x-x'\|^2_2)$. Then let $f(x_0):=\min_{x\in \text{supp}(X)} f(x)$ and set
	\begin{equation}\label{eq:monotone}
	F(x):= f(x_0) + \int_X (f(x)-f(x_0))\mathrm{d}x,
	\end{equation}
	then $F$ is an increasing function. 
	
	Based on a sample $(C_X,C_Y)$, we estimate
	\begin{align}
	p_x=\frac{1}{(mM)^2}\sum_{i=1}^{mM}\sum_{j=1}^{mM}\textbf{1}\{c_{y_j}\!<\!c_{x_i}\}.
	\end{align}
	We note that, since the GP has zero mean, its integral \eqref{eq:monotone} has a linear mean function.
	
	We introduce a \textit{confidence} in each decision, and this heuristic is near trivial when both algorithms return a probability $p_x$ (we set $p_y=1-p_x$). Thus we define confidence of a decision $d$ as
	\begin{equation}
	\text{conf}(d):=|p_x(d)-0.5|.
	\end{equation}
	We rank our decisions based on this: the higher the confidence, the more we believe in our decision. ANM and IGCI have other confidence scores \cite{benchmark}.
	
	When we state IGCI, we mean the slope-based estimator with uniform reference measure. For ANM, we applied the GP regression and the Hilbert-Schmidt Independence Criterion \cite{hsic}. Implementations are from the publicly available code given by \citet{benchmark}.  

	\section{Proof of Theorems}
	\renewcommand{\thesection}{\arabic{section}}
	\setcounter{section}{2}
	\begin{theorem}
		Let $X$ have support on a compact and connected subset $C$ of $\mathbb{R}$,
		and assume that $\mathbb{E}[Y|X=x]$ is a continuous differentiable function
		over $C$. Assume $\mathbb{E}Y^2 <\infty$. Let further $(x_i,y_i)$, $i=1,\ldots,N$,
		be iid samples from $\mathbb{P}_{(X,Y)}$. If we order, such that $x_{i+1}\geq x_i$ for all $i=1,\ldots,N-1$, then it holds that
		\begin{align}
		\frac{1}{N-1}\sum_{i=1}^{N-1}(y_{i+1}-y_{i})^2\rightarrow 2\mathbb{E}\textrm{\emph{Var}}(Y|X),
		\end{align}
		as $N\rightarrow \infty$.
	\end{theorem}
	\begin{proof}
		Let $f(x)\!:=\!\mathbb{E}[Y|X\!=\!x]$, and decompose for all $i$
		\begin{equation}
		y_i = f(x_i) + \big(y_i - f(x_i)\big)=:f(x_i) + \epsilon_i.
		\end{equation}
		Then we see that
		\begin{align}
		\sum_{i=1}^{N-1}\big(y_{i+1}-y_i\big)^2=&  \sum_{i=1}^{N-1}\big(f(x_{i+1})-f(x_{i})\big)^2 \nonumber \\ & + \sum_{i=1}^{N-1} \big(\epsilon_{i+1}-\epsilon_{i}\big)^2\\ & +2\sum_{i=1}^{N-1}\big(\epsilon_{i+1} -\epsilon_{i}\big)\big(f(x_{i+1})-f(x_i)\big),\nonumber
		\end{align}
		where the first and last terms tend to zero when scaled with $1/(N-1)$ due to Lemma~\ref{lem:1} (below) and the Cauchy-Schwartz inequality. Thus we are left with
		\begin{equation}
		\sum_{i=1}^{N-1} \big(\epsilon_{i+1}-\epsilon_{i}\big)^2 = \sum_{i=1}^{N-1}\epsilon_{i}^2 + \sum_{i=2}^{N}\epsilon_{i}^2 - 2\sum_{i=1}^{N-1}\epsilon_{i}\epsilon_{i+1},
		\end{equation}
		and the last term vanishes due to the iid assumption\footnote{Recall a sequence of iid variables, is still iid under any permutation.} and the fact that $\mathbb{E}\epsilon_{i} = 0$ for all $i$. Hence, as $N \rightarrow \infty$,
		\begin{align*}
		\frac{1}{N-1}&\sum_{i=1}^{N-1}\Big(y_i^2 + f(x_i)^2 - 2y_if(x_i)\Big)\\&\rightarrow \text{Var}(Y)+\text{Var}(\mathbb{E}[Y|X]) - 2\text{Cov}(Y,\mathbb{E}[Y|X])\\&=\text{Var}(Y)-\text{Var}(\mathbb{E}[Y|X])
		\\&=\mathbb{E}[\text{Var}(Y|X)],
		\end{align*}
		by the law of total variance\footnote{We assumed without loss of generality that $\mathbb{E}Y = 0$.}.
	\end{proof}

	\begin{proposition}
		Let $a,b,c,d\in\mathbb{R}$ and $a,c\neq0$. Assume $X$ and $Y$ are random variables with compact support. Then
		\begin{enumerate}
			\item[(1)] If we have that $\mathbb{E}[Y|X=x]= ax+ b$ and $\mathbb{E}[X|Y=y]=cy+d$, then $C_{X\rightarrow Y}=C_{Y\rightarrow X}$, in the limit of infinite data. 
			\item[(2)] We have $C_{aX+b \rightarrow cY + d}=C_{X\rightarrow Y}$.
		\end{enumerate}
	\end{proposition}
	\begin{proof}
		\emph{Ad (1):} 
		\begin{align*}
		C_{X\rightarrow Y}&\rightarrow \text{Corr}(\mathbb{E}[Y|X],Y)^2
		= \text{Corr}(aX+b,Y)^2 = \text{Corr}(X,Y)^2,
		\end{align*}
		and completely analogous for $C_{Y\rightarrow X}$.
		
		\emph{Ad (2):} 
		Clearly $\{x_i\}_{i=1,\ldots,N}$ and $\{ax_i + b\}_{i=1,\ldots,N}$ have the same sorting when $a\neq 0$.
		$C_{X\rightarrow Y}$ is obviously invariant to scaling and translating in $Y$, since we standardize the variable.
	\end{proof}
	\begin{proposition}
		If $X$ and $Y$ are random variables with compact support, and there exists
		measurable $f$ such that $Y=f(X)$, then $C_{X\rightarrow Y}\geq C_{Y_\rightarrow X}$,
		in the limit of infinite data.
	\end{proposition}
	\begin{proof}
		If there is no noise, then Lemma \ref{lem:1} suggests that $C_{X\rightarrow Y} \rightarrow 1$, which concludes the assertion in the limit.
	\end{proof}
	\renewcommand{\thesection}{\Alph{section}}
	\setcounter{theorem}{0}
	\begin{lemma} \label{lem:1}
		Let $X$ be a random variable with support on a compact and connected set $C\subset \mathbb{R}$ and let $f: \mathbb{R}\rightarrow\mathbb{R}$ be a continuously differentiable function over $C$. Let $x_i$ be independent samples of $X$ for $i=1,\ldots,N$. Then
		\begin{equation}
		\frac{1}{N-1}\sum_{i=1}^{N-1}\Big(f(x_{i+1})-f(x_{i})\Big)^2 \rightarrow 0\quad \text{as } N\rightarrow\infty,
		\end{equation}
		where $x_{i+1}\geq x_i$ for all $i$.
	\end{lemma}
	\begin{proof}
		For notation, we use $x_{(1)}\leq x_{(2)}\leq\ldots\leq x_{(N)}$ for the sorted sample.
		We denote 
		\begin{equation*}
		K_N = \frac{1}{N-1}\sum_{i=1}^{N-1}\Big(f(x_{(i+1)})-f(x_{(i)})\Big)^2.
		\end{equation*}
		Since $f$ is continuously differentiable, there exists $M:=\sup_{x\in C}f'(x)$, and by compactness there exists $a,b\in\mathbb{R}$, $a\leq b$, such that $C=[a,b]$. Then the bound, for any $N\geq 2$
		\begin{align}\label{unscaled}
		\begin{split}
		S_N:=M^2\Bigg((x_{(1)}-a)^2 + (b-x_{(N)})^2 + \sum_{i=1}^{N-1}\Big(x_{(i+1)}-x_{(i)}\Big)^2\Bigg)\geq K_N.
		\end{split}
		\end{align}
		Hence it suffices to show that for any $\epsilon>0$ there exists $N_0$, such that for all $N>N_0$, we have $S_N<\epsilon$. Naturally $S_N$ is downwards bounded by $0$, thus we may show that $S_N$ is a strictly descending sequence. See that for any fixed $N$ we have that $x_{N+1}\in [a,b]$, either $x_{N+1}\in[a,x_{(1)})$, $x_{N+1}\in[x_{(N)},b]$ or there exists some $j=1,\ldots,N-1$ such that $x_{N+1}\in[x_{(j)},x_{(j+1)})$. For the last case it holds that
		\begin{equation*}
		(x_{(j+1)}-x_{(j)})^2\geq (x_{(j+1)}-x_{N+1})^2+(x_{N+1}-x_{(j)})^2,
		\end{equation*}
		and the cases $a\leq x_{N+1}<x_{(1)}$ and $x_{(N)}\leq x_{N+1}\leq b$ follows analogously. This shows that $S_N>S_{N+1}$. Now scale $S_N$ with $\frac{1}{N-1}$ and observe that (\ref{unscaled}) still holds, hence $0\leq K_N\leq \frac{S_N}{N-1}\leq \frac{S_2}{N-1}\rightarrow 0$, and the assertion follows.
	\end{proof}

\bibliography{sources}

\begin{thebibliography}{19}
\providecommand{\natexlab}[1]{#1}
\providecommand{\url}[1]{\texttt{#1}}
\expandafter\ifx\csname urlstyle\endcsname\relax
  \providecommand{\doi}[1]{doi: #1}\else
  \providecommand{\doi}{doi: \begingroup \urlstyle{rm}\Url}\fi

\bibitem[Bl{\"o}baum et~al.(2018)Bl{\"o}baum, Janzing, Washio, Shimizu, and
  Sch{\"o}lkopf]{bloebaum}
Patrick Bl{\"o}baum, Dominik Janzing, Takashi Washio, Shohei Shimizu, and
  Bernhard Sch{\"o}lkopf.
\newblock Cause-effect inference by comparing regression errors.
\newblock In \emph{International Conference on Artificial Intelligence and
  Statistics}, pages 900--909, 2018.

\bibitem[Pearl(2009)]{pearl}
Judea Pearl.
\newblock \emph{Causality: models, reasoning, and inference}.
\newblock Cambridge University Press, 2009.

\bibitem[Hoyer et~al.(2009)Hoyer, Janzing, Mooij, Peters, and
  Sch\"{o}lkopf]{nonlinear}
Patrik~O. Hoyer, Dominik Janzing, Joris~M Mooij, Jonas Peters, and Bernhard
  Sch\"{o}lkopf.
\newblock Nonlinear causal discovery with additive noise models.
\newblock In D.~Koller, D.~Schuurmans, Y.~Bengio, and L.~Bottou, editors,
  \emph{Advances in Neural Information Processing Systems 21}, pages 689--696.
  Curran Associates, Inc., 2009.

\bibitem[Shimizu et~al.(2006)Shimizu, Hoyer, Hyvärinen, and Kerminen]{lingam}
Shohei Shimizu, Patrik~O. Hoyer, Aapo Hyvärinen, and Antti Kerminen.
\newblock A linear non-gaussian acyclic model for causal discovery.
\newblock \emph{Journal of Machine Learning Research}, 7:\penalty0 2003--2030,
  2006.

\bibitem[Zhang and Hyvärinen(2010)]{pnl}
Kun Zhang and Aapo Hyvärinen.
\newblock Distinguishing causes from effects using nonlinear acyclic causal
  models.
\newblock In Isabelle Guyon, Dominik Janzing, and Bernhard Schölkopf, editors,
  \emph{Proceedings of Workshop on Causality: Objectives and Assessment at NIPS
  2008}, volume~6 of \emph{Proceedings of Machine Learning Research}, pages
  157--164, Whistler, Canada, 12 Dec 2010. PMLR.

\bibitem[Sgouritsa et~al.(2015)Sgouritsa, Janzing, Hennig, and
  Schölkopf]{CURE}
Eleni Sgouritsa, Dominik Janzing, Philipp Hennig, and Bernhard Schölkopf.
\newblock Inference of cause and effect with unsupervised inverse regression.
\newblock In \emph{Proceedings of the 18th International Conference on
  Artificial Intelligence and Statistics (AISTATS)}, volume~38, pages 847--855,
  2015.

\bibitem[Janzing et~al.(2012)Janzing, Mooij, Zhang, Lemeire, Zscheischler,
  Daniusis, Steudel, and Schölkopf]{igci}
Dominik Janzing, Joris Mooij, Kun Zhang, Jan Lemeire, Jakob Zscheischler,
  Povilas Daniusis, Bastian Steudel, and Bernhard Schölkopf.
\newblock Information-geometric approach to inferring causal directions.
\newblock \emph{Artificial Intelligence}, \penalty0 (182-183):\penalty0 1--31,
  2012.

\bibitem[Mooij et~al.(2010)Mooij, Stegle, Janzing, Zhang, and Schölkopf]{gpi}
Joris Mooij, Oliver Stegle, Dominik Janzing, Kun Zhang, and Bernhard
  Schölkopf.
\newblock Probabilistic latent variable models for distinguishing between cause
  and effect.
\newblock In \emph{Advances in Neural Information Processing Systems},
  volume~23, pages 1687--1695, 01 2010.

\bibitem[Lawrence(2005)]{lawrence2005probabilistic}
Neil Lawrence.
\newblock Probabilistic non-linear principal component analysis with gaussian
  process latent variable models.
\newblock \emph{Journal of machine learning research}, 6\penalty0
  (Nov):\penalty0 1783--1816, 2005.

\bibitem[Peters et~al.(2017)Peters, Janzing, and Schölkopf]{elements}
Jonas Peters, Dominik Janzing, and Bernhard Schölkopf.
\newblock \emph{Elements of Causal Inference}.
\newblock MIT Press, 2017.

\bibitem[Daniusis et~al.(2010)Daniusis, Janzing, Mooij, Zscheischler, Steudel,
  Zhang, and Sch{\"o}lkopf]{daniusis}
P~Daniusis, D~Janzing, J~Mooij, J~Zscheischler, B~Steudel, K~Zhang, and
  B~Sch{\"o}lkopf.
\newblock Inferring deterministic causal relations.
\newblock In \emph{26th Conference on Uncertainty in Artificial Intelligence
  (UAI 2010)}, pages 143--150. AUAI Press, 2010.

\bibitem[Durrett(1996)]{durrett}
Richard Durrett.
\newblock \emph{Stochastic Calculus: A Practical Introduction}.
\newblock CRC Press, 1996.

\bibitem[Hall et~al.(1990)Hall, Kay, and Titterington]{diffbased}
Peter Hall, J.~W. Kay, and D.~M. Titterington.
\newblock Asymptotically optimal difference-based estimation of variance in
  nonparametric regression.
\newblock \emph{Biometrika}, 77\penalty0 (3):\penalty0 521--528, 1990.

\bibitem[Mooij et~al.(2016)Mooij, Peters, Janzing, Zscheischler, and
  Schölkopf]{benchmark}
Joris~M. Mooij, Jonas Peters, Dominik Janzing, Jakob Zscheischler, and Bernhard
  Schölkopf.
\newblock Distinguishing cause from effect using observational data: Methods
  and benchmarks.
\newblock \emph{Journal of Machine Learning Research}, \penalty0 (17):\penalty0
  1--102, 2016.

\bibitem[Jacod and Protter(2000)]{probess}
Jean Jacod and Philip Protter.
\newblock \emph{Probability Essentials}.
\newblock Springer, New York, 2000.

\bibitem[Gretton et~al.(2005)Gretton, Bousquet, Smola, and Schölkopf]{hsic}
Arthur Gretton, Olivier Bousquet, Alex Smola, and Bernhard Schölkopf.
\newblock Measuring statistical dependence with hilbert-schmidt norms.
\newblock \emph{Algorithmic Learning Theory}, pages 63--78, 2005.

\bibitem[Hall and Marron(1990)]{onvarest}
Peter Hall and J.~S. Marron.
\newblock On variance estimation in nonparametric regression.
\newblock \emph{Biometrika}, 77\penalty0 (2):\penalty0 415--419, 1990.

\bibitem[Zhou et~al.(2004)Zhou, Weston, Gretton, Bousquet, and
  Sch\"{o}lkopf]{rankonmanifolds}
Dengyong Zhou, Jason Weston, Arthur Gretton, Olivier Bousquet, and Bernhard
  Sch\"{o}lkopf.
\newblock Ranking on data manifolds.
\newblock In S.~Thrun, L.~K. Saul, and B.~Sch\"{o}lkopf, editors,
  \emph{Advances in Neural Information Processing Systems 16}, pages 169--176.
  MIT Press, 2004.

\bibitem[Silverman(1986)]{silverman}
Bernard~Walter Silverman.
\newblock \emph{Density Estimation for Statistics and Data Analysis}.
\newblock London: Chapman \& Hall/CRC, 1986.

\end{thebibliography}
\end{document}